\documentclass[english]{article}
\usepackage[T1]{fontenc}
\usepackage[latin9]{inputenc}
\usepackage{float}
\usepackage{amsmath}
\usepackage{amsthm}
\usepackage{amssymb}
\usepackage{graphicx}
\usepackage{comment}
\usepackage{geometry}
\usepackage{a4wide}

\floatstyle{ruled}
\newfloat{algorithm}{tbp}{loa}
\providecommand{\algorithmname}{Algorithm}
\floatname{algorithm}{\protect\algorithmname}

\theoremstyle{plain}
\newtheorem{thm}{\protect\theoremname}
  \theoremstyle{definition}
  \newtheorem{defn}[thm]{\protect\definitionname}

\usepackage{helvet}
\usepackage{courier}
\usepackage{algorithm,algorithmic}

\usepackage{times}
\usepackage[hyphens]{url}
\usepackage{babel}
  \providecommand{\definitionname}{Definition}
\providecommand{\theoremname}{Theorem}

\usepackage{authblk}
\title{Learning Transferable Domain Priors for Safe Exploration in Reinforcement
Learning}
\author[1]{Thommen George Karimpanal}
\author[1]{Santu Rana}
\author[1]{Sunil Gupta}
\author[1]{Truyen Tran}
\author[1]{Svetha Venkatesh}
\affil[1]{Applied Artificial Intelligence Institute, Deakin University, Australia}
\date{}                     
\begin{document}
\maketitle
\begin{abstract}
Prior access to domain knowledge could significantly improve the performance
of a reinforcement learning agent. In particular, it could help agents
avoid potentially catastrophic exploratory actions, which would otherwise
have to be experienced during learning. In this work, we identify
consistently undesirable actions in a set of previously learned tasks,
and use pseudo-rewards associated with them to learn a prior policy.
In addition to enabling safer exploratory behaviors in subsequent
tasks in the domain, we show that these priors are transferable to
similar environments, and can be learned off-policy and in parallel
with the learning of other tasks in the domain. We compare our approach
to established, \textit{state-of-the-art} algorithms in both discrete
as well as continuous environments, and demonstrate that it exhibits
a safer exploratory behavior while learning to perform arbitrary tasks
in the domain. We also present a theoretical analysis to support these
results, and briefly discuss the implications and some alternative
formulations of this approach, which could also be useful in certain
scenarios.
\end{abstract}

\section{Introduction}
Reinforcement learning (RL) \cite{sutton2011reinforcement} has proven
to be a versatile and powerful tool for effectively dealing with sequential
decision making problems. In addition to requiring only a scalar reward
feedback from the environment, its reliance on the knowledge of a
state transition model is limited. This has resulted in RL being successfully
used to solve a range of highly complex tasks \cite{Tesauro:1995:TDL:203330.203343,mnih2015human,silver2016mastering,Ng04invertedautonomous}.

However, RL algorithms are typically not sample efficient, and desired
behaviors are achieved only after the occurrence of several unsafe
agent-environment interactions, particularly during the initial phases
of learning. Even while operating within the same domain, commonly
undesirable actions (such as bumping into a wall in a navigation environment)
have to be learned to be avoided each time a new task (navigating
to a new goal location) is learned. This can largely be attributed
to the fact that in RL, behaviors are generally learned \textit{tabula-rasa}
(from scratch) \cite{dubey2018investigating}, without contextual
information of the domain it is operating in. This lack of contextual
knowledge is usually a limiting factor when it comes to deploying
RL algorithms in real world systems, where executing sub-optimal actions
during learning could be highly dangerous to the agent or to elements
in its environment. Providing RL agents with domain-specific contexts
in the form of suitable initializations and/or domain-specific, reusable
priors could greatly help mitigate this problem.

The challenge of addressing the issue of avoiding undesirable actions
during learning has been the primary focus of the field of safe RL
\cite{garcia2015comprehensive}, and consequently, a number of methods
have been proposed to enable RL agents to learn to solve tasks, with
due consideration given to the aspect of safety. These methods aim
to bias RL agents against such actions, broadly, by means of modifying
either the optimization criterion or the exploration process \cite{garcia2015comprehensive}.
In either case, the nature of the bias is to directly or indirectly
equip the agent with prior information regarding its domain, which
is subsequently used to enable safer learning behaviors. Safe RL approaches
where such prior knowledge is extracted from already learned tasks
in the domain share similarities with the ideology of transfer learning
\cite{taylor_transfer_2009}, in the sense that they both reuse previously
acquired knowledge to achieve a specific behavior. Perhaps the main
distinction between the two is that the former focuses on using domain-specific
knowledge to achieve safe behaviors, whereas the focus of the latter
is more generally, to reuse previously acquired task knowledge to
achieve good learning performance on a new task. Previous works \cite{fernandez_probabilistic_2006,li2018optimal,barreto2019transfer}
have explored the idea of exploiting known task knowledge for improving
learning performance, but ignore aspects relating to safety. Other
approaches which were specifically designed to enable safe exploration
\cite{achiam2017constrained,raybenchmarking,garcia2012safe} were
based on strong assumptions such as the availability of a safe baseline
policy or the explicit specification of a constraint function. Although
the idea of excluding unsafe actions during learning has been explored
in previous works \cite{alshiekh2018safe,zahavy2018learn}, they
too are reliant on explicit domain or safety specifications. In addition,
previous works that incorporate safe behaviors in RL agents have not
considered the issue of the ease of adaptation of the safe policy
in new, but related domains.

In this work, we propose an approach to learn a transferable domain
prior for safe exploration by incrementally extracting, refining and
reusing common domain knowledge from already learned policies, an
approach consistent with the ideology of continual learning \cite{ring1994continual}.
The reward function used for learning this prior is constructed by
approximating rewards from the $Q$- functions of the previously learned
tasks for state-action pairs consistently associated with undesirable
agent behaviors. Unlike other safe RL approaches, our approach does
not require the explicit specification of a safety or constraint function
to encode safe behaviors, or prior access to a safe policy. The focus
is to instead, extract knowledge from previously learned tasks to
learn a safety prior, which is subsequently used to bias an agent's
exploratory behavior while it learns arbitrary tasks in the domain.
The intuition behind this approach is that for a given domain, there
exist behaviors that are commonly undesirable for any arbitrary task
in that domain. As the prior is stored in the form of a $Q$- function,
it can be learned off-policy \cite{geist2014off}, in parallel with
an arbitrary task that the agent is learning, without the need for
additional interactions with the environment. The prior can also be
transferred or reused, and is capable of quickly adapting to other
similar environments, under the assumption that there exists a considerable
overlap in the set of undesirable actions in the two environments.
We demonstrate this claim in a simple tabular environment, while also
demonstrating the effectiveness of the proposed approach in more complex
environments with continuous states and/or continuous actions. We
also quantify the effectiveness of our approach in enabling safe exploration
in tabular domains by analytically deriving an expression that relates
the probability of executing unsafe actions using our approach, relative
to an $\epsilon$-greedy exploration strategy, for a given degree
of correctness of the learned priors. 

In summary, the main contributions of this work are: 

\begin{itemize}
\item A novel framework for learning domain priors from previously known tasks.
\item A theoretical relation between correctness of a prior and the relative probability of unsafe exploratory actions.
\item Experimental results in both discrete as well as continuous environments, validating the benefits of learning and using the described priors. 
\item Experimental results in the discrete action setting, demonstrating the transferability of the learned priors to other similar environments.
\end{itemize}

\section{Related Work}

The goal of our approach is to achieve safe exploratory actions during
the learning process by making use of existing knowledge of other
tasks in the domain, an ideology that is typical of many transfer
RL \cite{taylor_transfer_2009} frameworks. Specifically, we consider
the case where the tasks differ only in the reward functions \cite{barreto2017successor,ma2018universal}.
In one of the popular approaches \cite{fernandez_probabilistic_2006}
that addressed this case, past policies were reused based on their
similarity to the task being solved. In addition to being able to
effectively reuse past policies, the approach was also shown to be
capable of extracting a set of ``core'' policies to solve any task
in a given domain. A recent method by Li and Zhang \cite{li2018optimal}
improved this policy reuse approach by optimally selecting the source
policies online. However, these approaches, along with several others
\cite{schmitt2018kickstarting,spector2018sample} are only concerned
with the problem of reusing past policies to achieve quicker learning
in the target task, without consideration to the cost of executing
poor exploratory actions during learning. More recent works \cite{raybenchmarking,leike2017ai}
have emphasized this problem in greater detail, with accompanying environments
that demonstrate the distinction between reward-maximization behavior
and safety performance for a range of tasks. 

Most approaches that are directly concerned with achieving safe behaviors
during learning, do so by incorporating domain knowledge, and biasing
the actions of the learning agent by modifying either the optimization
criterion or the exploration process. A detailed summary of such approaches
can be found in Garcia and Fernandez \cite{garcia2015comprehensive}.
Among these, a few consider the problem of safety at the policy level
\cite{cohen2018diverse,ammar2015safe}, while others aim to improve
safety at the level of states and actions, much like the approach
described in the present work. The PI-SRL approach by Garcia and Fernandez
\cite{garcia2012safe} avoids the exploration of unsafe states by
using a known safe baseline policy, coupled with case-based reasoning.
However, the maintenance of their case-base of known states is based
on a Euclidean similarity metric, which may not be a useful measure
in many situations, and hence limits the generalizability of the approach.
Additionally, their assumption regarding the availability of a safe
baseline policy may not be reasonable in many practical circumstances.
The Lagrangian and constrained policy optimization approaches \cite{achiam2017constrained,raybenchmarking}
greatly improve safety performance. However, they require the explicit
specification of a safety performance metric or a constraint function,
which may not always be available. 

The idea of achieving safe learning behaviors by biasing against certain
actions has also been proposed in other recent work. Zahavy et al.
\cite{zahavy2018learn} proposed the approach of action elimination
deep $Q$-networks \cite{mnih2015human}, which essentially eliminates
sub-optimal actions, and performs $Q$-learning on a subset of the
state-action space. The elimination of actions is based on a binary
elimination signal which is computed using a contextual bandits framework.
Similar to this, the idea of shielding was proposed by Alshiekh et
al.\cite{alshiekh2018safe}, where unsafe actions were disallowed
based on a shielding signal. The authors synthesize the shield separately,
from a safety game between an environment and a system player. Akin
to these approaches, the basis of our approach is to bias the agent
against certain actions that are considered to be undesirable, as
per a learned prior policy. However, the key idea is to obviate the
need for domain-specific safety constraints, and instead, learn a
safety prior from a set of previously learned tasks, in an online
and off-policy manner, without the requirement of additional interactions
with the environment. 

\section{Methodology\label{sec:Methodology}}

We consider the objective of learning a prior policy $\pi_{P}$ by
learning the corresponding $Q$-function $Q_{P}$ in a domain $\mathcal{D=<S,A,T>}$,
where the tasks $\mathcal{M=<D,R>}$ share a common state-space $\mathcal{S}$,
action-space $\mathcal{A}$ and state-transition function $\mathcal{T}$,
and differ solely in the reward function $\mathcal{R}$. The purpose
of this prior is to bias the agent against exploratory actions that
have a high degree of undesirability , which we define as follows:
\begin{defn}
\label{def:undesirability}The undesirability of an action $a$ is
the absolute value of the optimal advantage $A^{*}(s,a)$ for that
action, where $A^{*}(s,a)=Q^{*}(s,a)-\underset{a'\in\mathcal{A}}{max}\thinspace Q^{*}(s,a')$.
\end{defn}
The optimal advantage function $A^{*}(s,a)$ \cite{baird1993advantage}
measures the deviation of the $Q$ -value for a particular state-action
pair $(s,a)$ from the maximum $Q$ -value associated with the state
$s$. Thus, $|A^{*}(s,a)|$ is indicative of how much worse action
$a$ is, in relation to the best action in that state. 

In order to learn $Q_{P}$, we assume that we know the optimal $Q$
-functions corresponding to $N$ arbitrary tasks in the domain $\mathcal{D}$.
For the sake of argument, let us consider the case where $N>1$, which
implies there exist at least a few tasks $\boldsymbol{M}=\{\mathcal{M}_{1}...\mathcal{M}_{i}...\mathcal{M}_{N}\}$
whose optimal $Q$- functions $\boldsymbol{Q^{*}=}\{Q_{1}^{*}...Q_{i}^{*}...Q_{N}^{*}\}$
are known. In the proposed approach, $Q_{P}$ corresponds to a pseudo-task
$\mathcal{\mathcal{M_{P}}=<D,R_{P}>}$ that is learned off-policy
by sampling state-action pairs in the given domain, for example, by
executing random exploratory actions in the environment. More practically,
they are sampled as per a behavior policy $\pi_{B}$ corresponding
to an arbitrary task $\mathcal{\mathcal{M_{\text{\ensuremath{\Omega}}}}=<D,R_{\text{\ensuremath{\Omega}}}>}$
, that is being learned in parallel. Although in general, any off-policy
approach could be used to learn $Q_{P}$, for simplicity, here, we
show the learning of $Q_{P}$ using $Q$-learning \cite{watkins1989learningfrom}. 

The basis of our approach is to construct the pseudo-reward function
$\mathcal{R_{P}}$ based on state-action pairs that are consistently
undesirable across the $N$ known tasks. We infer rewards that would
likely be associated with such state-action pairs and subsequently
construct $\mathcal{R_{P}}$ as a weighted sum of these inferred rewards.
Once $\mathcal{R_{P}}$ is constructed, $Q_{P}$ is learned off-policy,
and is subsequently used to bias the exploratory actions of the agent.
Corresponding to this description, our methodology is composed of
the following steps:

\subsection{Identification of Suitable State-Action Pairs}

The first step in our approach is to identify state-action pairs that
are consistently associated with undesirable agent behaviors. Once
a state-action pair $(s,a)$ has been sampled using the agent's behavior
policy $\pi_{B}$, for each task $\mathcal{M}_{i}$ of the $N$ known
tasks, we measure the undesirability $w_{i}(s,a)$ of the action as
a quantity proportional to the action's undesirability, as per Definition
\ref{def:undesirability}. In order to scale these values to be $\leq1$
, we measure the scaled undesirability $w_{i}(s,a)$ as: 
\begin{equation}
w_{i}(s,a)=\left|\frac{A_{i}^{*}(s,a)}{\underset{a'\in\mathcal{A}}{max}\thinspace Q_{i}^{*}(s,a')}\right|\label{eq:weighted_adv}
\end{equation}
We repeat this procedure for each of the $N$ tasks, and store the
obtained measures in a sequence $W(s,a)$ as follows:

\begin{equation}
W(s,a)=\left\{ w_{1}(s,a),...w_{i}(s,a),...w_{N}(s,a)\right\} 
\end{equation}


The overall consensus on the undesirability of action $a$ in state
$s$, as per the $N$ known tasks can then be measured by quantifying
the consistency in the values stored in $W(s,a)$. We do this by converting
$W(s,a)$ into a probability distribution $W^{'}(s,a)$ and then measuring
the normalized entropy $\mathcal{H}(W^{'}(s,a))$ associated with
it:

\begin{equation}
\mathcal{H}(W'(s,a))=-\frac{\sum_{i=1}^{N}w'_{i}(s,a)\thinspace log(w_{i}^{'}(s,a))}{log(N)}\label{eq:entropyeqn}
\end{equation}

\noindent where $W^{'}(s,a)=\{w_{1}^{'}(s,a),...w_{i}^{'}(s,a),...w_{N}^{'}(s,a)\}$,
and $w_{i}^{'}(s,a)$, the $i^{th}$ element of $W^{'}(s,a)$, is
computed using the softmax function:
\begin{equation}
w_{i}^{'}(s,a)=\frac{e^{w_{i}(s,a)}}{\sum_{i=1}^{N}e^{w_{i}(s,a)}}\label{eq:wdash}
\end{equation}


In order to construct the pseudo-reward function $\mathcal{R_{P}}$,
we select state-action pairs which are associated with high values
of $w_{i}(s,a)$, as well as a high normalized entropy value $\mathcal{H}(W^{'}(s,a))$.
The former criterion, quantified by the mean $\mu(W(s,a))=\frac{\sum_{i=1}^{N}w_{i}(s,a)}{N}$
of the values in $W(s,a)$, prioritizes state-action pairs that are
highly undesirable. The latter criterion $\mathcal{H}(W^{'}(s,a))$
quantifies the consistency of the undesirability of the state-action
pair across the known tasks. To account for both these criteria, we
use a threshold $t$, and select state-action pairs for which: 
\begin{equation}
\mathcal{H}(W'(s,a))*\mu(W(s,a))>t\label{eq:selection}
\end{equation}


The general idea is to select state-action pairs associated with highly
and consistently undesirable behaviors across the known tasks in the
domain. The selection of state-action pairs using Equation \ref{eq:selection}
depends heavily on the choice of a suitable threshold value $t$,
for which a rough guideline can be obtained by considering the ranges
of $\mathcal{H}(W'(s,a))$ and $\mu(W(s,a))$. $H(W'(s,a))$ lies
in the range $[0,1]$, while the range of $\mu(W(s,a))$ depends on
that of the function $\left|\frac{A^{*}(s,a)}{\underset{a'\in\mathcal{A}}{max}\thinspace Q^{*}(s,a')}\right|$,
or equivalently, using Definition \ref{def:undesirability}, $\left|\frac{Q^{*}(s,a)-\underset{a'\in\mathcal{A}}{max}\thinspace Q^{*}(s,a')}{\underset{a'\in\mathcal{A}}{max}\thinspace Q^{*}(s,a')}\right|$.
The minimum value of this function is $0$, which corresponds to the
case when $\underset{a'\in\mathcal{A}}{a=argmax}\thinspace Q^{*}(s,a')$.
The maximum value corresponds to the case when $Q^{*}(s,a)$ is as
low as possible, and $\underset{a'\in\mathcal{A}}{max}\thinspace Q^{*}(s,a')$
is as large as possible. If $r_{min}$ and $r_{max}$ represent the
lowest and highest possible rewards in the domain, then using the
lower and upper bounds of $\frac{r_{min}}{1-\gamma}$ and $\frac{r_{max}}{1-\gamma}$
for the $Q$- function, the maximum possible value of $\left|\frac{A_{i}^{*}(s,a)}{\underset{a'\in\mathcal{A}}{max}\thinspace Q_{i}^{*}(s,a')}\right|$
would be: $\left|\frac{r_{min}-r_{max}}{r_{max}}\right|$. Hence,
threshold $t$ must be selected to be in the range $[0,\left|\frac{r_{min}-r_{max}}{r_{max}}\right|]$.
In general, a lower threshold value results in a larger number of
state-action pairs being selected for the construction of $\mathcal{R_{P}}$,
possibly leading to a more conservative prior.

\subsection{Constructing Pseudo-rewards and Learning $Q_{P}$\label{subsec:Constructing-pseudo-rewards-and}}

The next step is to use the identified state-action pairs to construct
a safety prior. Consider an arbitrary task $\mathcal{M}$ in the domain
for which the policy is learned using $Q$- learning. The corresponding
standard update equation is given by:

\begin{equation}
Q(s,\negthinspace a)\leftarrow Q(s,\negthinspace a)+\alpha[r(s,a,s')+\gamma\thinspace\underset{a'\in\mathcal{A}}{max\thinspace}Q(s',\negthinspace a')-Q(s,\negthinspace a)]\label{eq:qlearning}
\end{equation}

Here, $s$ and $a$ represent the current state and action, $\gamma$
is the discount factor $(0\leq\gamma\leq1)$, $s'$ is the next state,
and $r(s,a,s')$ is the reward associated with the transition.

When the optimal $Q$-function $Q^{*}$ is learned, the temporal difference
(TD) error: $[r(s,a,s')+\gamma\thinspace\underset{a'\in\mathcal{A}}{max}\thinspace Q^{*}(s',a')-Q^{*}(s,a)]$
would reduce to $0$. Using this fact, we can infer the original reward
$r(s,a,s')$ associated with the transition:

\begin{equation}
r(s,a,s')=Q^{*}(s,a)-\gamma\thinspace\underset{a'\in\mathcal{A}}{max}\thinspace Q(s',a')\label{eq:infer_reward}
\end{equation}

In reality, the above equality seldom holds, as the TD error may not
be exactly $0$. However, the inferred reward may still be a reasonable
approximation if the $Q$-function is close to optimal ($Q\thickapprox Q^{*}$).
With this assumption in mind, we apply Equation \ref{eq:infer_reward}
to each of the known tasks, and construct the rewards associated with
those state-action pairs $(s_{c},a_{c})$ which satisfy the condition
in Equation \ref{eq:selection}. The pseudo-reward $r_{P}$ is computed
as a sum of these inferred rewards, weighted by the corresponding
elements of $W'(s_{c},a_{c})$:

\begin{equation}
r_{p}(\negthinspace s_{c},\negthinspace a_{c},\negthinspace s_{c}')\negthinspace=\negthinspace\sum_{i=1}^{N}w'_{i}(s_{c},\negthinspace a_{c})[Q_{i}^{*}(s_{c},\negthinspace a_{c})-\gamma\thinspace\underset{a'\in\mathcal{A}}{max}\thinspace Q_{i}^{*}(s_{c}',\negthinspace a')]\label{eq:weightedrewards}
\end{equation}

$r_{P}$ is capped to have a maximum absolute value of $1$, and for
state-action pairs that do not satisfy Equation \ref{eq:selection},
$r_{P}$ is set to a default value of $0$. $r_{P}$ is then used
to update the $Q$- function $Q_{P}$ via the standard $Q$- learning
update equation (Equation \ref{eq:qlearning}). By continuously sampling
state-action pairs, determining the corresponding pseudo-reward $r_{P}$
and updating $Q_{P}$, the optimal $Q$- function $Q_{P}^{*}$, is
learned. It is worth mentioning that $Q_{P}$ is updated using what
ever state-action pairs are sampled by the behavior policy $\pi_{B}$.
Hence, no additional interactions with the environment are required
for its computation. However, learning $Q_{P}^{*}$ is subject to
the condition that $\pi_{B}$ sufficiently explores the state-action
space. The additional requirements for learning a prior policy are
the additional memory and computations corresponding to inferring
$r_{P}$, and storing and updating $Q_{P}$. The overall process of
updating $Q_{P}$ is summarised in Algorithm \ref{alg:algorithm1}.

\begin{algorithm}[h]
\caption{Algorithm for updating prior $Q$-function $Q_{P}$}
\begin{algorithmic}[1]

\STATE \textbf{Input: }

\STATE Set of $N$ optimal $Q$- functions $\boldsymbol{Q^{*}=}\{Q_{1}^{*}...Q_{i}^{*}...Q_{N}^{*}\}$,
Estimate of prior $Q$-function $Q_{P}$, maximum number of steps
per episode $H$, behavior policy $\pi_{B}$, threshold $t$

\STATE \textbf{Output: }updated estimate of $Q_{P}$\textbf{ }

\FOR { H steps}

\STATE Execute behavior policy $\pi_{B}$ to take action $a$ from
state $s$, and obtain next state\textbf{ $s'$}

\STATE Initialize $W(s,a)$ as an empty set

\FOR { each task i of the N known tasks}

\STATE Compute $A_{i}^{*}(s,a)=Q_{i}^{*}(s,a)-\underset{a'\in\mathcal{A}}{max}\thinspace Q_{i}^{*}(s,a')$

\STATE $w_{i}(s,a)=\left|\frac{A_{i}^{*}(s,a)}{\underset{a'\in\mathcal{A}}{max}\thinspace Q_{i}^{*}(s,a')}\right|$

\STATE $W(s,a)=W(s,a)\cup w_{i}(s,a)$

\ENDFOR 

\STATE Normalize $W(s,a)$ using Equation \ref{eq:wdash} to obtain
$W^{'}(s,a)=\{w_{1}^{'}(s,a)...w_{i}^{'}(s,a)...w_{N}^{'}(s,a)\}$ 

\STATE Compute $\mathcal{H}(W^{'}(s,a))$ (Equation \ref{eq:entropyeqn})

\STATE Compute $\mu(W(s,a))=\frac{1}{N}\sum_{i=1}^{N}w_{i}(s,a)$

\STATE Initialize pseudo-reward $r_{P}(s,a,s')$ as $0$

\IF{$\mu(W(s,a))*\mathcal{H}(W^{'}(s,a))>t$ (threshold)}

\STATE${r_{P}(s,a,s')}$\\${=\sum_{i=1}^{N}w'_{i}(s,a)[Q_{i}^{*}(s,a)-\gamma\thinspace\underset{a'\in\mathcal{A}}{max}\thinspace Q_{i}^{*}(s',a')]}$ 

\ENDIF

\STATE${Q_{P}(s,a)\longleftarrow Q_{P}(s,a)}$ ${+\alpha[r_{P}(s,a,s')+\gamma \underset{a'\in\mathcal{A}}{max}Q_{P}(s',a')-Q_{P}(s,a)]}$ 

\ENDFOR 

\end{algorithmic}

\label{alg:algorithm1}
\end{algorithm}

\subsection{Biasing Exploration Using $Q_{P}^{*}$}

Following the construction of the domain priors, the final step is
to use these priors to bias the exploratory behavior of the agent.
$Q_{P}^{*}$ is learned based on a reward function $\mathcal{R_{P}}$,
which is specifically constructed using state-action pairs that are
consistently associated with undesirable actions. Hence, in order
to avoid catastrophic actions during learning, we simply bias the
agent's behavior against taking undesirable actions, as determined
by $Q_{P}^{*}$. If such an action happens to be suggested by the
agent during learning, with a high probability $\rho$, we disallow
it from being executed, and force the agent to pick an alternative
action whose $Q_{P}^{*}$ value is at least equal to the mean value
of $Q_{P}^{*}$ over all actions. The threshold of $\underset{a'\in\mathcal{A}}{mean}{Q_{P}^{*}}(s,a')$
was chosen, simply to ensure that better-than-average actions are
executed during exploration. More conservative (higher) or radical
(lower) threshold values could also be considered, although it must
be noted that choosing a very high threshold would limit the extent
of exploration, while a very low threshold would fail to leverage
the safe exploratory behaviors enabled by $Q_{P}^{*}$. Algorithm
\ref{alg:algorithm2} outlines the process of biasing the agent against
undesirable exploratory actions.

\begin{algorithm}[h]
\caption{Biasing against undesirable exploration}
\begin{algorithmic}[1]

\STATE \textbf{Input: }

\STATE Proposed exploratory action $a_{0}$, state $s$, optimal
$Q$- function of prior $Q_{P}^{*}$, probability of using priors
$\rho$

\STATE \textbf{Output: }selected action $a$\textbf{ }

\STATE With a probability $\rho$:\textbf{ }

\WHILE{${Q^*_P}(s,a)<\underset{a'\in\mathcal{A}}{mean}{Q^*_P}(s,a')$}

\STATE Pick random action from $\mathcal{A}:$ $a_{0}=random(\mathcal{A})$ 

\ENDWHILE

\STATE$a=a_{0}$

\end{algorithmic}

\label{alg:algorithm2}
\end{algorithm}

\section{Theoretical Analysis}

Biasing the exploratory actions as described would, in an ideal case,
help avoid unsafe actions. However, the effectiveness of using the
learned priors to bias against these actions is highly dependent on
how correct the priors are. In this section, we consider the discrete
actions setting, and derive a relation between the correctness of
a prior and the probability of taking unsafe actions using our approach,
relative to an $\epsilon$-greedy exploration policy. We first define
the terms `unsafe actions' and `correctness of a prior' for the purpose
of our analysis, as follows:
\begin{defn}
\label{def:unsafe_actions}An action $a$ is considered unsafe in
a state $s$ if ${\ensuremath{{Q_{P}^{*}}(s,a)<\underset{a'\in\mathcal{A}}{mean}{Q_{P}^{*}}(s,a')}}$
in that state.
\end{defn}
Definition \ref{def:unsafe_actions} was chosen to be consistent with
the biasing criteria used in Algorithm \ref{alg:algorithm2}. 
\begin{defn}
The correctness $C_{Q_{P},\mathcal{D}}$ of a prior $Q_{P}$, with
respect to a domain $\mathcal{D}$ is the probability with which it
avoids deeming an action to be safe, when it is actually unsafe.

$C_{Q_{P},\mathcal{D}}=1-\frac{n_{FN}}{n_{I}-n_{FP}+n_{FN}}$
\end{defn}
where $n_{FP}$ and $n_{FN}$ are respectively the number of false
positives (cases where the action has been incorrectly classified
by $Q_{P}$ as unsafe) and false negatives (cases where the action
has been incorrectly classified by $Q_{P}$ as safe), and $n_{I}$
is the number of unsafe actions identified by $Q_{P}$. It is worth
noting that only the false negative cases affect the probability of
encountering truly unsafe actions. The effect of false positives would
be to simply slow down learning. The extent to which the correctness
$C_{Q_{P},\mathcal{D}}$ affects the probability of encountering unsafe
actions, relative to the case of $\epsilon$-greedy exploration, is
presented in the following theorem:
\begin{thm}
If a prior $Q_{P}$ with a correctness of $C_{Q_{P},\mathcal{D}}$,
is used to bias the exploratory actions with a probability of $\rho$,
then relative to the case of standard $\epsilon$-greedy exploration,
the probability of executing unsafe exploratory actions in a given
state is reduced by a factor of $1-\frac{\rho(|\mathcal{A}|C_{Q_{P,\mathcal{\mathcal{D}}}}-U)}{|\mathcal{A}|-U}$,
where $\mathcal{A}$ is the action space associated with the domain,
and $U$ is the number of unsafe actions associated with that state.
\end{thm}
\begin{proof}
For the case of standard $\epsilon-greedy$ exploration, the agent
takes exploratory actions with a probability of $\epsilon$, in each
instance of which, the probability of picking an unsafe action is
$\frac{U}{|\mathcal{A}|}$. Hence, the probability of unsafe exploratory
actions for an $\epsilon-greedy$ strategy is: $p_{\epsilon-greedy}=\frac{\epsilon U}{|\mathcal{A}|}$

Now, in the case of biased exploration, exploratory actions occur
with a probability of $\epsilon,$ and are biased using the priors,
with a probability $\rho$. When the bias is used, the agent eliminates
unsafe actions (as determined by $Q_{P}$), and uniformly and randomly
selects from the remaining $|\mathcal{A}|-U$ actions. However, the
selected action may still be unsafe due to the presence of false negatives,
which occur with a probability of $1-C_{Q_{P},\mathcal{D}}$. With
the remaining probability of $(1-\rho)$, exploration occurs exactly
as in the $\epsilon-greedy$ case. Hence, the total probability of
unsafe actions occuring during exploring is: $p_{priors}=\frac{\epsilon U\rho(1-C_{Q_{P}})}{|\mathcal{A}|-U}+\frac{\epsilon U(1-\rho)}{|\mathcal{A}|}$.
The ratio $\frac{p_{priors}}{p_{\epsilon-greedy}}$ can then be simplified
to: $1-\frac{\rho(|\mathcal{A}|C_{Q_{P}}-U)}{|\mathcal{A}|-U}$
\end{proof}
This implies that fewer unsafe actions can be expected when $C_{Q_{P},\mathcal{D}}$
and $\rho$ have values close to $1$. Although a large value of $\rho$
is favorable, in order to maintain a non-zero probability of visiting
every state-action pair (and thus ensure convergence), it is set to
be slightly lesser than $1$. For the purpose of this analysis, we
only considered environments with discrete actions. However, in practice,
our approach was also used to bias exploration in continuous action
environments in Section \ref{subsec:AI-Safety-Gym}. This was done
by randomly sampling a large number of actions from a uniform distribution,
and applying the exploration bias on this set of discretized actions. 

\section{Results}

\subsection*{Benchmark Environments and Baselines}

In order to test the learning and safety performance of the described
approach, we chose three different environments. The first is a classical
navigation environment shown in Figure \ref{fig:environments}(a),
first introduced by Fernandez and Veloso \cite{fernandez_probabilistic_2006},
where the state and action spaces are discrete. For this tabular environment,
we use OPS-TL\cite{li2018optimal}, PRQL\cite{fernandez_probabilistic_2006},
PI-SRL\cite{garcia2012safe} and $Q$- learning\cite{watkins1989learningfrom}
as baselines for performance comparison. 

Next, we show the agent's performance in a safety grid world `Island
Navigation' environment, shown in Figure \ref{fig:environments}(b),
which was first introduced by Leike et al. \cite{leike2017ai} as
a benchmark designed to evaluate safe exploration performance. The
choice of baselines for this environment was A2C\cite{mnih2016asynchronous},
SARSA\cite{sutton2011reinforcement} and DQN\cite{mnih2015human}. 

Lastly, we demonstrate the performance of our approach on a safe exploration
task, shown in Figure \ref{fig:environments}(c), in the `Safety Gym'
environment, a continuous action environment recently introduced by
Ray et al. \cite{raybenchmarking}. For this environment, the chosen
baselines were PPO\cite{schulman2017proximal}, PPO-Lagrangian (a
version of PPO with explicit constraints\cite{raybenchmarking})
and DDPG\cite{lillicrap2015continuous}.

We chose to validate our approach using these selected environments,
as typical RL tasks in environments such as Atari \cite{mnih2013playing}
or OpenAI Gym \cite{brockman2016openai} are set up largely with
a focus on learning performance, without much consideration given
to aspects relating to safety.

\begin{figure*}[t]
\begin{centering}
\includegraphics[scale=0.25]{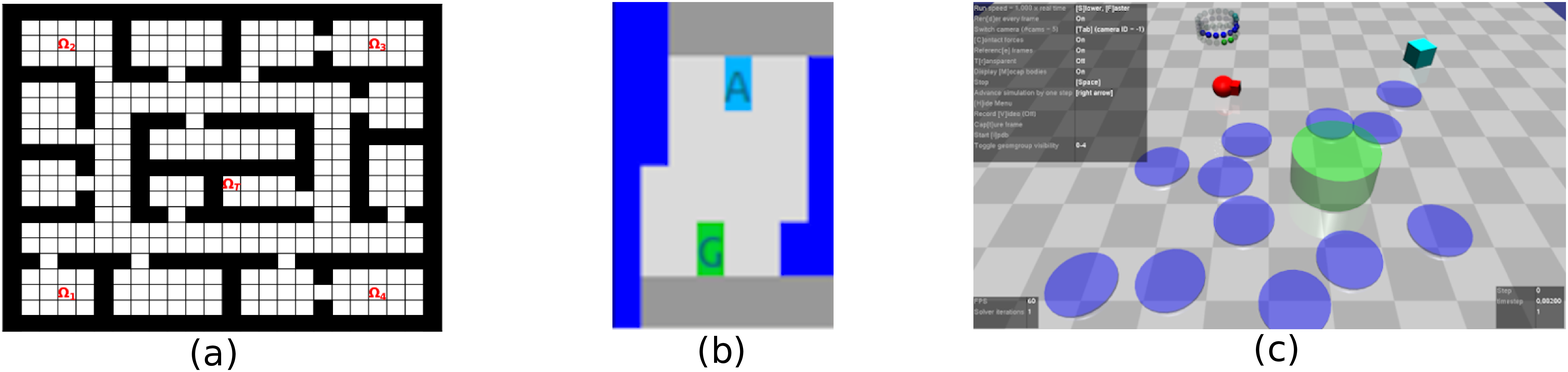}
\par\end{centering}
\caption{(a) shows the classical navigation environment, with goal locations
$\text{\ensuremath{\Omega}}_{1},\text{\ensuremath{\Omega}}_{2},\text{\ensuremath{\Omega}}_{3},\text{\ensuremath{\Omega}}_{4}$
of the known tasks, and goal location $\text{\ensuremath{\Omega}}_{\text{{T}}}$
of the task to be learned. (b) shows the agent `A', the goal `G' and
the `water' locations in blue in the island navigation environment
and (c) shows a task in the Safety Gym PointGoal environment, where
the green area is the navigation target, and the purple areas represent
hazards which need to be avoided by the agent (in red).}
\label{fig:environments}
\end{figure*}

\subsection{Classical Navigation Environment\label{subsec:Classical-Navigation-Environment}}

We first demonstrate extensive results from our approach on a classical
$21\times24$ grid-world navigation environment shown in Figure \ref{fig:environments}(a),
before proceeding to more complex and continuous environments in Sections
\ref{subsec:AI-Safety-Gridworld} and \ref{subsec:AI-Safety-Gym}.
The environment settings are consistent with those reported in \cite{fernandez_probabilistic_2006}.
Here, each state is represented by a $1\times1$ grid cell, with darker
colored cells representing obstacles, and other cells representing
free positions. The agent's state is represented by its $(x,y)$ coordinates,
and at each state, it is allowed to take one of four actions - moving
up, down, left or right. Following the execution of an action, the
agent moves to a new state, which is noised by random values sampled
from a uniform distribution in the range (-0.2,0.2). 

When the agent executes an action that causes it to bump into an obstacle,
it retains its original state, without moving and receives a reward
of $-1$. Goal states are terminal, and transitions leading into them
are associated with a reward of $1$. For all other transitions, the
agent receives a small negative reward of $-0.1$. This penalises
behaviors such as moving back and forth between two non-goal states. 

For each task, the agent is allowed to interact with the environment
for $K$ episodes. Each episode starts with the agent in a random,
non-goal state, following which, it could execute upto $H$ actions
to try and reach the terminal goal state. The performance $W$ of
the agent is evaluated by computing the discounted sum of rewards
per episode as follows:

\begin{equation}
W=\frac{1}{K}\sum_{k=0}^{K}\sum_{h=0}^{H}\gamma^{h}r_{k,h}\label{eq:performancemeasure}
\end{equation}

\noindent where $r_{k,h}$ is the reward received from the environment
at step $h$ of episode $k$. We use the same metric to evaluate the
performances in the continuous environments.

In order to obtain source policies, the agent is initially trained
to learn the tasks $\mathcal{M}_{\text{\ensuremath{\Omega}}_{1}},\mathcal{M}_{\text{\ensuremath{\Omega}}_{2}},\mathcal{M}_{\text{\ensuremath{\Omega}}_{3}}$
and $\mathcal{M}_{\text{\ensuremath{\Omega}}_{4}}$ , corresponding
to the navigation target locations $\Omega_{1},\Omega_{2},\Omega_{3}$
and $\Omega_{4}$. The label $\text{\ensuremath{\Omega}}_{\text{\text{{T}}}}$
in Figure \ref{fig:environments}(a) marks the goal location of the
target task $\mathcal{M}_{\text{\ensuremath{\Omega}}_{\text{{T}}}}$,
which the agent aims to learn. 

The prior is learned using the optimal $Q$-functions of tasks $\mathcal{M}_{\text{\ensuremath{\Omega}}_{1}},\mathcal{M}_{\text{\ensuremath{\Omega}}_{2}},\mathcal{M}_{\text{\ensuremath{\Omega}}_{3}}$
and $\mathcal{M}_{\text{\ensuremath{\Omega}}_{4}}$, as described
in Algorithm \ref{alg:algorithm1}. Figure \ref{fig:worstactions_orig}
depicts the set of consistently undesirable actions identified using
these known tasks, which is then used for learning the prior $Q_{P}$.
The red, green, blue and orange arrows represent actions that move
the agent up, right, down and left respectively. As observed in Figure
\ref{fig:worstactions_orig}, most of the identified actions correspond
to those that would cause collisions with obstacles in the environment.
The task $\mathcal{M}_{\text{\ensuremath{\Omega}}_{\text{{T}}}}$
is then learned by biasing the exploratory actions of the agent using
the learned prior, as described in Algorithm \ref{alg:algorithm2}.

\begin{figure}[H]
\centering{}\includegraphics[width=0.6\columnwidth]{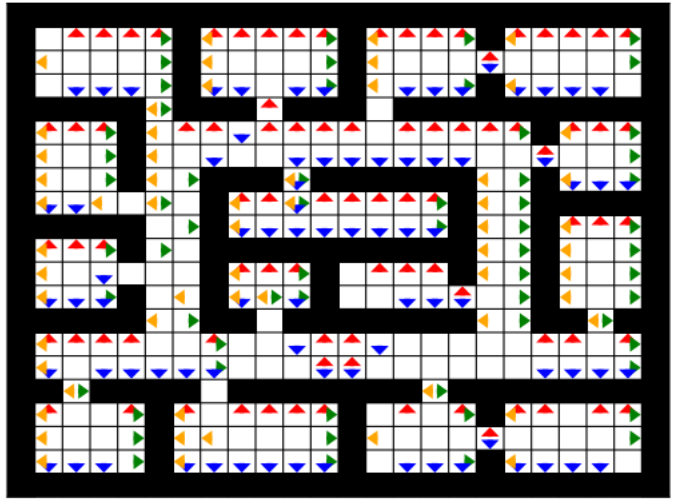}\caption{Identified set of consistently undesirable actions extracted from
known tasks for the environment in Figure \ref{fig:environments}(a).}
\label{fig:worstactions_orig}
\end{figure}

Figure \ref{fig:performance_learning} shows the average performance
over $10$ trials, of different algorithms, evaluated using Equation
\ref{eq:performancemeasure}. The shaded regions represent the standard
errors of the mean performances for the $10$ trials. The common learning
parameters were set as follows: $\alpha=0.05,\gamma=0.95,H=500,K=2000$,
and the probability of exploration $\epsilon$ was set to be decaying
from an initial value of $1$, as in \cite{fernandez_probabilistic_2006}.
Two of the performance curves in Figures \ref{fig:performance_learning}
and \ref{fig:performance_failures} were obtained by combining the
described approach with: a) standard $Q$-learning \cite{watkins1989learningfrom},
and b) PRQ-learning (PRQL) \cite{fernandez_probabilistic_2006}$(\psi=1,\nu=0.95)$.
The parameters specific to our approach were chosen to be: $t=0.35,\rho=0.95$.
As observed from the figure, these curves exhibit a superior learning
and safety performance compared to their corresponding counterparts,
in which the learning occurs without the use of domain priors. In
particular, the use of learned priors enables a significant increase
in the initial performance of the agent, due to fewer unsafe exploratory
actions during the initial phases of learning. This is supported by
the results in Figure \ref{fig:performance_failures}, which depicts
the trend in the number of obstacle collisions per episode in each
of the tested approaches. The overall performance of the agent is
also superior to that of other approaches such as OPS-TL \cite{li2018optimal}$(c=0.0049)$
for selecting source tasks, and the PI-SRL approach \cite{garcia2012safe}
$(k=6,\sigma=0.5)$, in which safe exploratory actions are chosen
based on case-based reasoning. Although the latter approach has a
marginally better initial performance as seen in Figure \ref{fig:performance_learning},
the learned policy is very conservative, as indicated by the negligible
improvement in its performance across the episodes. From these figures,
it is evident that the use of domain priors brings about improvements
in both safety as well as learning performance.

\begin{figure}[H]
\centering{}\includegraphics[width=0.8\columnwidth]{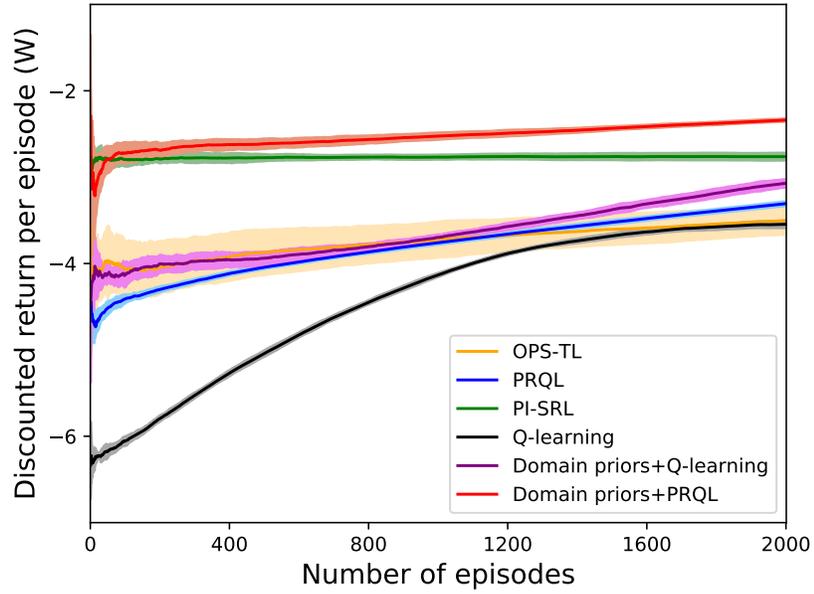}\caption{The average discounted returns per episode ($W$), computed over $10$
trials, for different learning methods in the classical navigation
environment.}
\label{fig:performance_learning}
\end{figure}

\begin{figure}[H]
\centering{}\includegraphics[width=0.8\columnwidth]{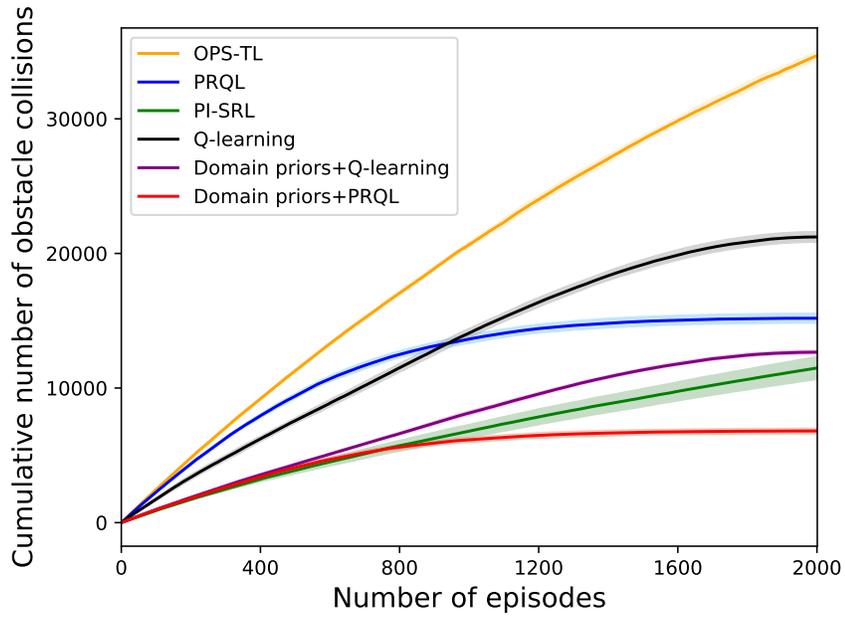}\caption{The cumulative number of obstacle collisions, computed over $10$
trials, for different learning methods in the classical navigation
environment.}
\label{fig:performance_failures}
\end{figure}

\subsection{Continuous State Environment\label{subsec:AI-Safety-Gridworld}}

The results from Section \ref{subsec:Classical-Navigation-Environment}
demonstrate the effectiveness of the proposed method in simple tabular
domains. Although the nature of the task in the non-tabular `Island
Navigation' domain \cite{leike2017ai} considered in this section
is roughly similar to that in Section \ref{subsec:Classical-Navigation-Environment},
there exists a fundamental difference between the two, in that the
states are now represented using features. The goal in this environment
is for the agent to navigate to the target location using a set of
discrete actions (moving left,right,up and down) without stepping
into the `water' locations. In order to obtain the source policies
to construct the priors, we first solved a set of $4$ random tasks
using Deep $Q$-learning (DQN) \cite{mnih2015human} by randomly
generating the target locations. Consistent with the implementation
in Leike at al. \cite{leike2017ai}, both the A2C as well as the
DQN implementations used a $2$ layered multi-layer perceptrons with
$100$ nodes each, trained with inputs that consisted of a matrix
encoding the current configuration of the environment. The architecture
for SARSA was kept identical to that for DQN, and varied only in
the value function update rule. For A2C, we used an entropy penalty
parameter of $0.05$, which linearly decayed to $0$ at the end of
each trial. For optimization, we used Adam\cite{kingma2014adam}
with a learning rate of $5e-4$ and a batch size of
$64$. For each task, the agent was trained for $2000$ episodes,
each consisting of up to $100$ steps. The other parameters used were:
a discount factor of $0.99$, an initial exploration parameter of
$1$, which decayed exponentially to a minimum of $0.1$ (with a decay
factor of 0.95), a replay buffer of size $2000$, a threshold $t=0.25$
and $\rho=0.95$.

Using the obtained source policies, we implemented our approach described
in Section \ref{sec:Methodology}, and tested the performance of the
agent on a new task, while its exploration was biased using the learned
priors, as described in Algorithm \ref{alg:algorithm2}. Figures \ref{fig:performance_learning-1}
and \ref{fig:performance_failures-1} depict the performance of various
approaches, averaged over $15$ trials. As observed, our method of
biasing the exploration using the learned priors was able to improve
the agent's learning performance, while simultaneously achieving a
fewer visits to the `water' locations, thereby also improving the
safety performance.

\begin{figure}[H]
\centering{}\includegraphics[width=0.8\columnwidth]{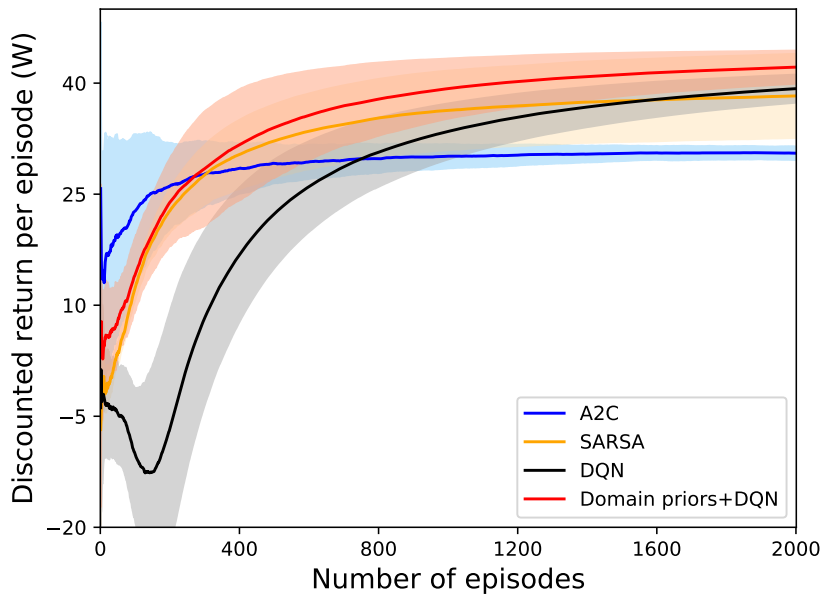}\caption{The average discounted returns per episode ($W$), computed over $15$
trials, for different learning methods in the island navigation environment.}
\label{fig:performance_learning-1}
\end{figure}

\begin{figure}[H]
\centering{}\includegraphics[width=0.8\columnwidth]{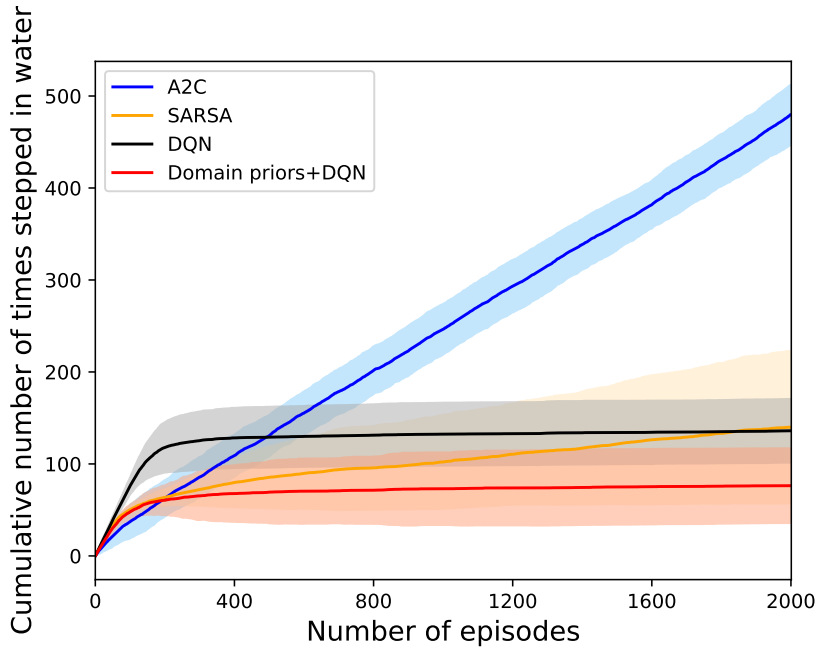}\caption{The cumulative number of times stepped in water, computed over $15$
trials, for different learning methods in the island navigation environment.}
\label{fig:performance_failures-1}
\end{figure}

Our method for learning domain priors naturally scales to such non-tabular
environments, fundamentally because the process of inferring $r_{P}$
(Equation \ref{eq:weightedrewards}) does not explicitly depend on
the state complexity, and only depends on the $Q$-values of the $N$
tasks for the specific transition $(\negthinspace s_{c},\negthinspace a_{c},\negthinspace s_{c}')$
under consideration. This can be obtained with a maximum of $N|\mathcal{A}|$
queries to the stored $Q$-networks, which depends only on $|\mathcal{A}|$
and $N$, and is independent of the size of the state space. 

\subsection{Continuous Action Environment\label{subsec:AI-Safety-Gym}}

The `Safety Gym' \cite{raybenchmarking} environment consists of
both continuous states and actions. To implement our approach in such
as setting, we chose a version of the PointGoal1 environment, `PointGoal1-12',
where the number of `hazards' were set to $12$, making it a more
unsafe environment than the original PointGoal1 environment. The aim
of the agent in this environment is to navigate to the goal location
while avoiding the `hazard' locations. Each of the $1000$ episodes
are run for $1000$ steps. As in the case of the other environments,
we initially obtained source policies by separately training a DDPG
\cite{lillicrap2015continuous} agent on $3$ tasks. Using these
source policies, we implemented our described approach for safe exploration.
For the DDPG implementation, the critic and target networks were multi-layer
perceptrons with $3$ and $2$ layers respectively, with the former
having $1024,512$ and $300$ nodes in its three layers, and the latter
with $512$ and $128$ nodes in its layers. The learning rates for
both networks were set to $1e-4$, the soft target update parameter
$\tau$ was set to $1e-2$, the discount factor was set to $0.99$
and the replay buffer size was set to be $100000$. For PPO, the hyperparameters
used were consistent with those used in Ray et al. \cite{raybenchmarking}.
The hyperparameters specific to the approach described here are $\rho=0.95$
and threshold $t=0.1$.

\begin{figure}[H]
\centering{}\includegraphics[width=0.8\columnwidth]{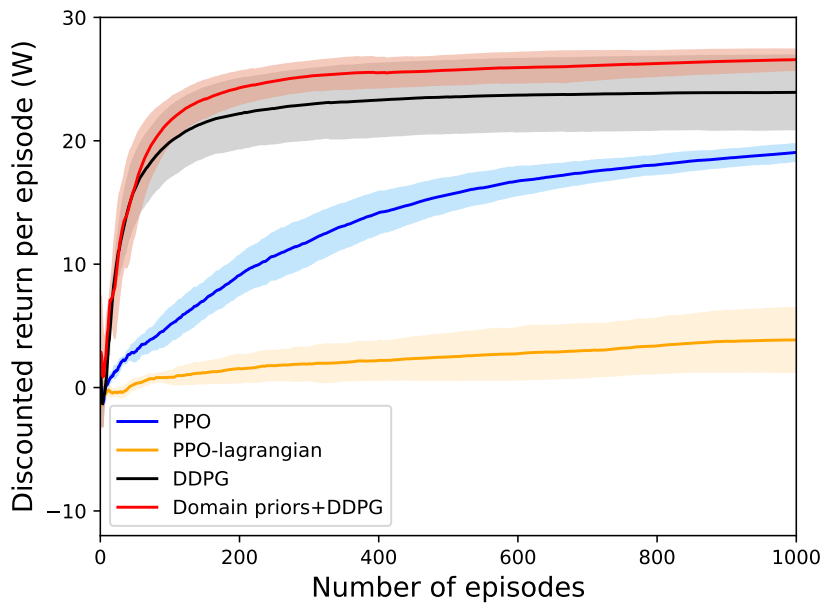}\caption{The average discounted returns per episode ($W$), computed over $3$
trials, for different learning methods in the PointGoal1-12 Safety
Gym environment.}
\label{fig:performance_learning-2}
\end{figure}

\begin{figure}[H]
\centering{}\includegraphics[width=0.8\columnwidth]{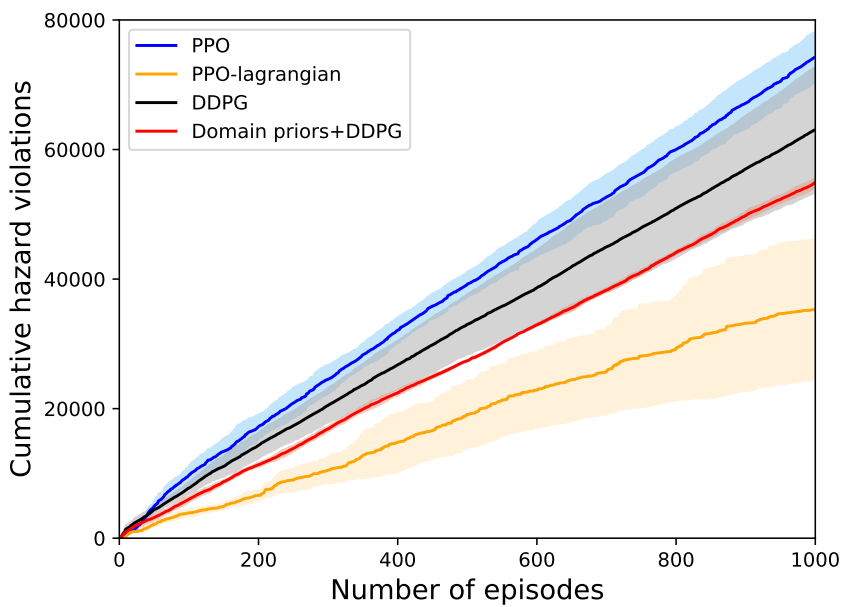}\caption{The cumulative number of obstacle collisions, computed over $3$ trials,
for different learning methods in the PointGoal1-12 Safety Gym environment.}
\label{fig:performance_failures-2}
\end{figure}

As the environment contains a continuous action space, biasing the
exploration exactly as described in Algorithm \ref{alg:algorithm2}
is infeasible. In order to circumvent this issue, we randomly sampled
$100$ actions from a uniform distribution in the allowable range
of actions, $(-1,1)$, essentially discretizing the action space.
Following this, we proceeded to bias the actions as per Algorithm
\ref{alg:algorithm2}. The actions were biased with a probability
proportional to an exploration bias factor, which started with an
initial value of $1$, and decayed exponentially by a factor of $0.95$
at the end of each episode.

As depicted in Figures \ref{fig:performance_learning-2} and \ref{fig:performance_failures-2},
the use of priors helps improve both learning as well as safety performance.
As also noted in Ray et al. \cite{raybenchmarking}, although the
learning performance of the PPO-Lagrangian approach is poor, it exhibits
a much superior safety performance. However, it must be pointed out
that this method has explicit access to a constraint violation function,
while our approach does not. 

\subsection{Prior Adaptation to Modified Environments}

\begin{figure*}[t]
\begin{centering}
\includegraphics[scale=0.8]{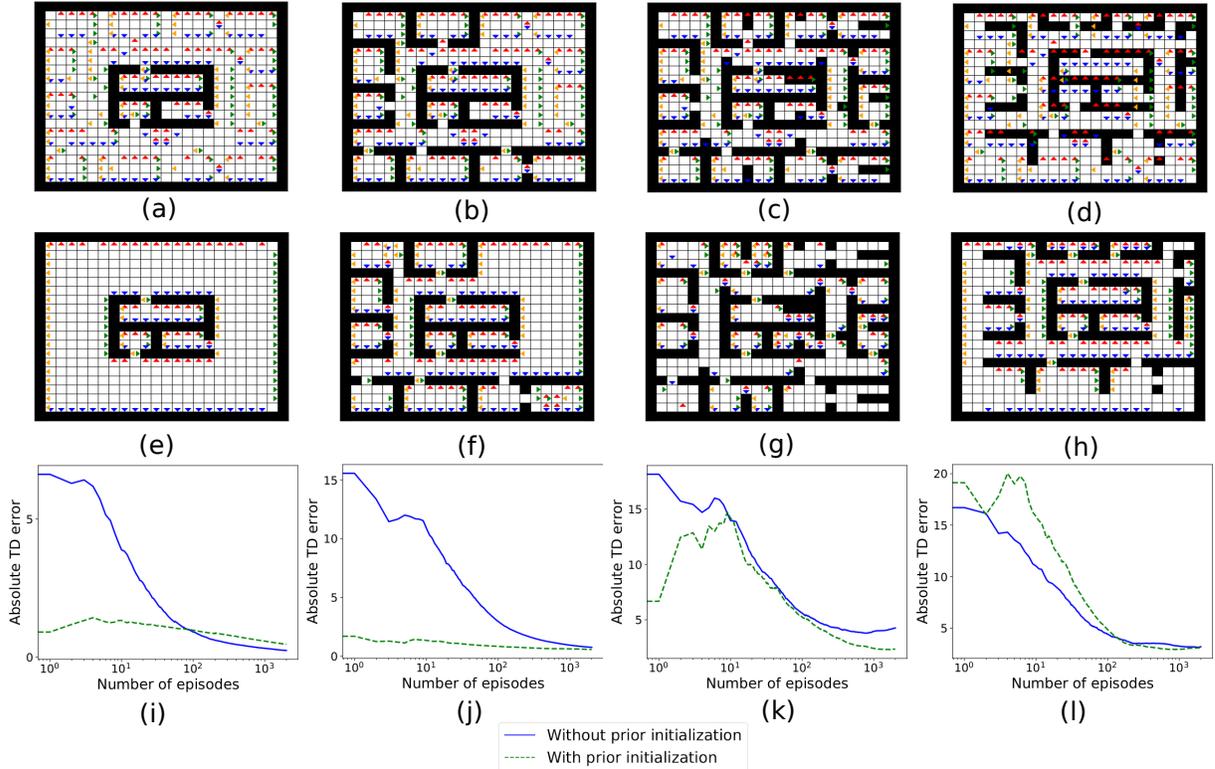}
\par\end{centering}
\caption{(a)-(d) show the consistently undesirable actions corresponding to
the original environment in Figure \ref{fig:environments}(a), overlaid
on top of four modified environments. (e)-(h) show these environments,
with actions that are actually undesirable in them. (i)-(l) show the
absolute TD errors associated with the learning of $Q_{P}$ for these
environments, with and without prior initialization.}
\label{fig:adapting_priors}
\end{figure*}

As shown in Sections \ref{subsec:Classical-Navigation-Environment},\ref{subsec:AI-Safety-Gridworld}
and \ref{subsec:AI-Safety-Gym}, learned priors can effectively help
avoid undesirable exploratory actions while learning an arbitrary
task in the domain. However, if the environment was to undergo a change
in configuration, the set of actions associated with unsafe agent
behaviors would not remain the same. Nevertheless, provided these
changes are not too drastic, the priors learned from the original
environment could still serve as a useful initialization for learning
the corresponding priors in the modified environment. In other words,
the priors may be transferable to the modified environments. This
is an advantage that is specific to our approach, and is enabled by
the fact that our priors are adaptive, and are inherently tied to
the structure of the domain. In addition, the adaptive nature of the
priors ensures that in time, they become well-suited to the modified
environment, the with the adaptation time depending on the degree
of dissimilarity between the two environments.

Here, we design experiments in the tabular environment in Section
\ref{subsec:Classical-Navigation-Environment}, to demonstrate this
transferability to modified versions of the original environment in
Figure \ref{fig:environments}(a), shown in Figures \ref{fig:adapting_priors}(a)-(d).
Obstacles were either added or removed from the original environment
(Figure \ref{fig:environments}(a)) to obtain the modified environments
in Figures \ref{fig:adapting_priors}(a)-(c), whereas the environment
in Figure \ref{fig:adapting_priors}(d) was created by offsetting
most obstacles 2 units upwards and to the right. The consistently
undesirable actions for the original environment in Figure \ref{fig:environments}(a)
are overlaid on top of the modified environments in Figures \ref{fig:adapting_priors}(a)-(d),
whereas the correct set of consistently undesirable actions for the
modified environments are shown in Figures \ref{fig:adapting_priors}(e)-(h).
Despite the differences between the undesirable actions of the original
and modified environments, there exists some structural similarity
between them. Hence, it is reasonable to expect the priors learned
in the original environment to be at least partially transferable
to the modified environments. Specifically, we posit that the learned
prior for the original environment forms a reasonable initial estimate
for learning the corresponding priors in the modified environments,
as long as the differences between the two are not drastic. 

In order to test this hypothesis, the priors for the modified environments
were learned with and without these initial estimates. In both cases,
the associated absolute TD errors decrease, as shown in Figures \ref{fig:adapting_priors}(i)-(l),
which demonstrates the capability of the priors to adapt to different
environments. Figures \ref{fig:adapting_priors}(i)-(k) suggest that
initialization of the priors could lead to significantly lowered initial
absolute TD errors compared to the case of learning the priors from
scratch (without initialization). However, initializing the priors
in this manner was not found to be useful for the environment in Figure
\ref{fig:adapting_priors}(d), where the effect of the initialization
was to slightly increase the initial absolute TD error, as depicted
in Figure \ref{fig:adapting_priors}(l). This is due to the fact that
the nature of the differences in the obstacle configuration in Figure
\ref{fig:adapting_priors}(d) and Figure \ref{fig:environments}(a)
renders the prior learned in the latter ineffective with respect to
learning the prior in the former. These experiments demonstrate that
while the prior learned using the described approach is transferable
to some extent, it is not transferable in general.

\section{Discussion}

The proposed methodology allows RL agents avoid undesirable actions
during learning by making use of a learned prior policy. Although
our approach as described, deals with avoiding undesirable actions,
it can be easily adapted to scenarios where there exist actions that
are commonly desirable across the tasks in the domain. Such an adaptation
would involve replacing the advantage $A_{i}^{*}(s,a)$ with $B_{i}^{*}(s,a)=Q_{i}^{*}(s,a)-\underset{a'\in\mathcal{A}}{min}Q_{i}^{*}(s,a')$,
in addition to replacing Equation \ref{eq:weighted_adv} with $w_{i}(s,a)=\left|\frac{B_{i}^{*}(s,a)}{\underset{a'\in\mathcal{A}}{max}\thinspace Q_{i}^{*}(s,a')}\right|$.
The resulting prior could then simply be used to guide exploration,
by taking exploratory actions that are greedy with respect to $Q_{P}^{*}$
with a high probability. Such an approach appeared to be successful
in versions of the tabular environment (similar to that described
in Section \ref{subsec:Classical-Navigation-Environment}) where a
non-goal, rewarding state was introduced into all tasks in the domain.
Although the approach is useful for such specific situations, in general,
exploring the state-action space by greedily exploiting the prior
in this manner could lead to poor learning performances, as it may
limit the agent's exploration. Hence, achieving safe learning behaviors
is a more practical use-case for the approach described in this work. 

The ability to avoid undesired actions during learning makes the proposed
approach potentially useful for real-world systems which are often
intolerant of poor actions. Our approach would thus be useful in scenarios
where the associated marginal increase in memory and computational
costs are outweighed by the costs of executing unsafe actions. 

Although we only consider cases where tasks vary solely in the reward
function, this could lay the foundation for more general work, where
tasks vary in other aspects such as the representation, transition
function or the state-action space. 

\section{Conclusion}

We presented a method to extract priors from a set of known tasks
in the domain. The prior is learned in the form of a $Q$-function,
and is based on inferred rewards corresponding to consistently undesirable
actions across these tasks. The effectiveness of the prior in enabling
safe learning behaviors was demonstrated in discrete as well as continuous
environments, and its performance was compared to various baselines.
This was further supported by our theoretical analysis, which suggests
that the use of these priors helps reduce the probability of taking
unsafe exploratory actions. In addition to leading to safer learning
behaviors for arbitrary tasks in the domain, the priors were shown
to be transferable to some extent, and capable of adapting to changes
in the environment.

\bibliographystyle{plain}
\bibliography{example}

\begin{thebibliography}{10}

\bibitem{achiam2017constrained}
Joshua Achiam, David Held, Aviv Tamar, and Pieter Abbeel.
\newblock Constrained policy optimization.
\newblock In {\em Proceedings of the 34th International Conference on Machine
  Learning-Volume 70}, pages 22--31. JMLR. org, 2017.

\bibitem{alshiekh2018safe}
Mohammed Alshiekh, Roderick Bloem, R{\"u}diger Ehlers, Bettina K{\"o}nighofer,
  Scott Niekum, and Ufuk Topcu.
\newblock Safe reinforcement learning via shielding.
\newblock In {\em Thirty-Second AAAI Conference on Artificial Intelligence},
  2018.

\bibitem{ammar2015safe}
Haitham~Bou Ammar, Rasul Tutunov, and Eric Eaton.
\newblock Safe policy search for lifelong reinforcement learning with sublinear
  regret.
\newblock In {\em International Conference on Machine Learning}, pages
  2361--2369, 2015.

\bibitem{baird1993advantage}
Leemon~C Baird.
\newblock Advantage updating.
\newblock Technical report, Wright Lab Wright-Patterson AFB OH, 1993.

\bibitem{barreto2019transfer}
Andr{\'e} Barreto, Diana Borsa, John Quan, Tom Schaul, David Silver, Matteo
  Hessel, Daniel Mankowitz, Augustin {\v{Z}}{\'\i}dek, and Remi Munos.
\newblock Transfer in deep reinforcement learning using successor features and
  generalised policy improvement.
\newblock {\em arXiv preprint arXiv:1901.10964}, 2019.

\bibitem{barreto2017successor}
Andr{\'e} Barreto, Will Dabney, R{\'e}mi Munos, Jonathan~J Hunt, Tom Schaul,
  Hado~P van Hasselt, and David Silver.
\newblock Successor features for transfer in reinforcement learning.
\newblock In {\em Advances in neural information processing systems}, pages
  4055--4065, 2017.

\bibitem{brockman2016openai}
Greg Brockman, Vicki Cheung, Ludwig Pettersson, Jonas Schneider, John Schulman,
  Jie Tang, and Wojciech Zaremba.
\newblock Openai gym.
\newblock {\em arXiv preprint arXiv:1606.01540}, 2016.

\bibitem{cohen2018diverse}
Andrew Cohen, Lei Yu, and Robert Wright.
\newblock Diverse exploration for fast and safe policy improvement.
\newblock In {\em Thirty-Second AAAI Conference on Artificial Intelligence},
  2018.

\bibitem{dubey2018investigating}
Rachit Dubey, Pulkit Agrawal, Deepak Pathak, Tom Griffiths, and Alexei Efros.
\newblock Investigating human priors for playing video games.
\newblock In {\em International Conference on Machine Learning}, pages
  1348--1356, 2018.

\bibitem{fernandez_probabilistic_2006}
Fern\'ando Fern\'andez and Manuela Veloso.
\newblock Probabilistic policy reuse in a reinforcement learning agent.
\newblock In {\em Proceedings of the fifth international joint conference on
  {Autonomous} agents and multiagent systems}, pages 720--727. ACM, 2006.

\bibitem{garcia2012safe}
Javier Garcia and Fernando Fern{\'a}ndez.
\newblock Safe exploration of state and action spaces in reinforcement
  learning.
\newblock {\em Journal of Artificial Intelligence Research}, 45:515--564, 2012.

\bibitem{garcia2015comprehensive}
Javier Garc{\i}a and Fernando Fern{\'a}ndez.
\newblock A comprehensive survey on safe reinforcement learning.
\newblock {\em Journal of Machine Learning Research}, 16(1):1437--1480, 2015.

\bibitem{geist2014off}
Matthieu Geist and Bruno Scherrer.
\newblock Off-policy learning with eligibility traces: a survey.
\newblock {\em Journal of Machine Learning Research}, 15(1):289--333, 2014.

\bibitem{kingma2014adam}
Diederik~P Kingma and Jimmy Ba.
\newblock Adam: A method for stochastic optimization.
\newblock {\em arXiv preprint arXiv:1412.6980}, 2014.

\bibitem{leike2017ai}
Jan Leike, Miljan Martic, Victoria Krakovna, Pedro~A Ortega, Tom Everitt,
  Andrew Lefrancq, Laurent Orseau, and Shane Legg.
\newblock Ai safety gridworlds.
\newblock {\em arXiv preprint arXiv:1711.09883}, 2017.

\bibitem{li2018optimal}
Siyuan Li and Chongjie Zhang.
\newblock An optimal online method of selecting source policies for
  reinforcement learning.
\newblock In {\em Thirty-Second AAAI Conference on Artificial Intelligence},
  2018.

\bibitem{lillicrap2015continuous}
Timothy~P Lillicrap, Jonathan~J Hunt, Alexander Pritzel, Nicolas Heess, Tom
  Erez, Yuval Tassa, David Silver, and Daan Wierstra.
\newblock Continuous control with deep reinforcement learning.
\newblock {\em arXiv preprint arXiv:1509.02971}, 2015.

\bibitem{ma2018universal}
Chen Ma, Junfeng Wen, and Yoshua Bengio.
\newblock Universal successor representations for transfer reinforcement
  learning.
\newblock {\em arXiv preprint arXiv:1804.03758}, 2018.

\bibitem{mnih2016asynchronous}
Volodymyr Mnih, Adria~Puigdomenech Badia, Mehdi Mirza, Alex Graves, Timothy~P
  Lillicrap, Tim Harley, David Silver, and Koray Kavukcuoglu.
\newblock Asynchronous methods for deep reinforcement learning.
\newblock In {\em International Conference on Machine Learning}, 2016.

\bibitem{mnih2013playing}
Volodymyr Mnih, Koray Kavukcuoglu, David Silver, Alex Graves, Ioannis
  Antonoglou, Daan Wierstra, and Martin Riedmiller.
\newblock Playing atari with deep reinforcement learning.
\newblock {\em arXiv preprint arXiv:1312.5602}, 2013.

\bibitem{mnih2015human}
Volodymyr Mnih, Koray Kavukcuoglu, David Silver, Andrei~A Rusu, Joel Veness,
  Marc~G Bellemare, Alex Graves, Martin Riedmiller, Andreas~K Fidjeland, Georg
  Ostrovski, et~al.
\newblock Human-level control through deep reinforcement learning.
\newblock {\em Nature}, 518(7540):529, 2015.

\bibitem{Ng04invertedautonomous}
Andrew~Y. Ng, H.~Jin Kim, Michael~I. Jordan, and Shankar Sastry.
\newblock Inverted autonomous helicopter flight via reinforcement learning.
\newblock In {\em International Symposium on Experimental Robotics}. MIT Press,
  2004.

\bibitem{raybenchmarking}
Alex Ray, Joshua Achiam, and Dario Amodei.
\newblock Benchmarking safe exploration in deep reinforcement learning.

\bibitem{ring1994continual}
Mark~Bishop Ring.
\newblock {\em Continual learning in reinforcement environments}.
\newblock PhD thesis, University of Texas at Austin Austin, Texas 78712, 1994.

\bibitem{schmitt2018kickstarting}
Simon Schmitt, Jonathan~J Hudson, Augustin Zidek, Simon Osindero, Carl Doersch,
  Wojciech~M Czarnecki, Joel~Z Leibo, Heinrich Kuttler, Andrew Zisserman, Karen
  Simonyan, et~al.
\newblock Kickstarting deep reinforcement learning.
\newblock {\em arXiv preprint arXiv:1803.03835}, 2018.

\bibitem{schulman2017proximal}
John Schulman, Filip Wolski, Prafulla Dhariwal, Alec Radford, and Oleg Klimov.
\newblock Proximal policy optimization algorithms.
\newblock {\em arXiv preprint arXiv:1707.06347}, 2017.

\bibitem{silver2016mastering}
David Silver, Aja Huang, Chris~J Maddison, Arthur Guez, Laurent Sifre, George
  Van Den~Driessche, Julian Schrittwieser, Ioannis Antonoglou, Veda
  Panneershelvam, Marc Lanctot, et~al.
\newblock Mastering the game of go with deep neural networks and tree search.
\newblock {\em Nature}, 529(7587):484--489, 2016.

\bibitem{spector2018sample}
Benjamin Spector and Serge Belongie.
\newblock Sample-efficient reinforcement learning through transfer and
  architectural priors.
\newblock {\em arXiv preprint arXiv:1801.02268}, 2018.

\bibitem{sutton2011reinforcement}
Richard~S Sutton and Andrew~G Barto.
\newblock Reinforcement learning: An introduction, 2011.

\bibitem{taylor_transfer_2009}
Matthew~E Taylor and Peter Stone.
\newblock Transfer learning for reinforcement learning domains: A survey.
\newblock {\em Journal of Machine Learning Research}, 10(Jul):1633--1685, 2009.

\bibitem{Tesauro:1995:TDL:203330.203343}
Gerald Tesauro.
\newblock Temporal difference learning and td-gammon.
\newblock {\em Commun. ACM}, 38(3):58--68, March 1995.

\bibitem{watkins1989learningfrom}
CJCH Watkins.
\newblock Learningfrom delayed rewards.
\newblock {\em PhDthesis, Cambridge University, Cambridge, England}, 1989.

\bibitem{zahavy2018learn}
Tom Zahavy, Matan Haroush, Nadav Merlis, Daniel~J Mankowitz, and Shie Mannor.
\newblock Learn what not to learn: Action elimination with deep reinforcement
  learning.
\newblock In {\em Advances in Neural Information Processing Systems}, pages
  3562--3573, 2018.

\end{thebibliography}

\newpage

\section*{Supplementary material:}

\section{Application to common reward case}

In domains where there exists a common, non-terminal rewarding state
$s_{com}$, the proposed approach can be modified to positively bias
the agent towards taking greedy actions with respect to the learned
prior $Q_{P}$, as described in the discussion section. By doing so,
we shift the focus of the algorithm to finding consistently desirable
actions across the known tasks in this common reward environment.
Here, we present one such environment, where in addition to the attributes
of the environment in Figure \ref{fig:environments} (a), there exists
a non-terminal rewarding state $s_{com}$ associated with a reward
of $0.2$, shown in Figure \ref{fig:commonrew_env}. In such a case, visiting state $s_{com}$ becomes a desirable
behavior across all tasks. Hence, the learned prior directs learning
agents towards this state, as seen in Figure \ref{fig:commonrew_priors}.
Such a bias in the exploration policy is also reflected in the performance
of the agent, as depicted in Figure \ref{fig:performance_commonrew}.

\begin{figure}[H]
\centering{}\includegraphics[width=0.6\columnwidth]{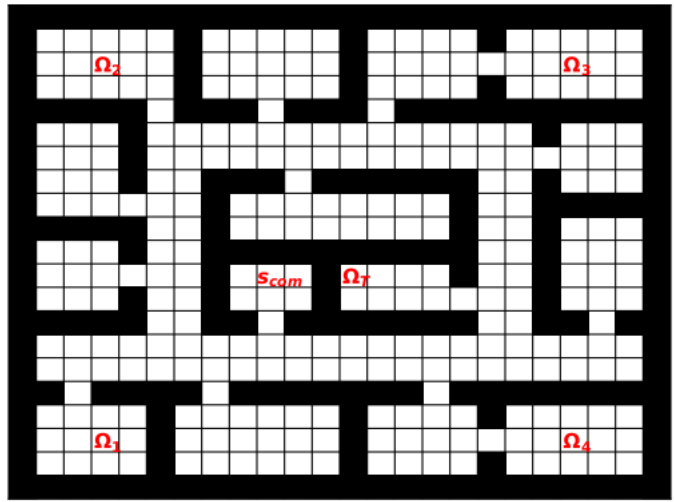}\caption{Navigation environment showing the goal locations $\text{\ensuremath{\Omega}}_{1},\text{\ensuremath{\Omega}}_{2},\text{\ensuremath{\Omega}}_{3},\text{\ensuremath{\Omega}}_{4}$
of the known tasks, common rewarding state $s_{com}$ and goal location
$\text{\ensuremath{\Omega}}_{\text{{T}}}$ of the task to be learned. }
\label{fig:commonrew_env}
\end{figure}

\begin{figure}[H]
\centering{}\includegraphics[width=0.6\columnwidth]{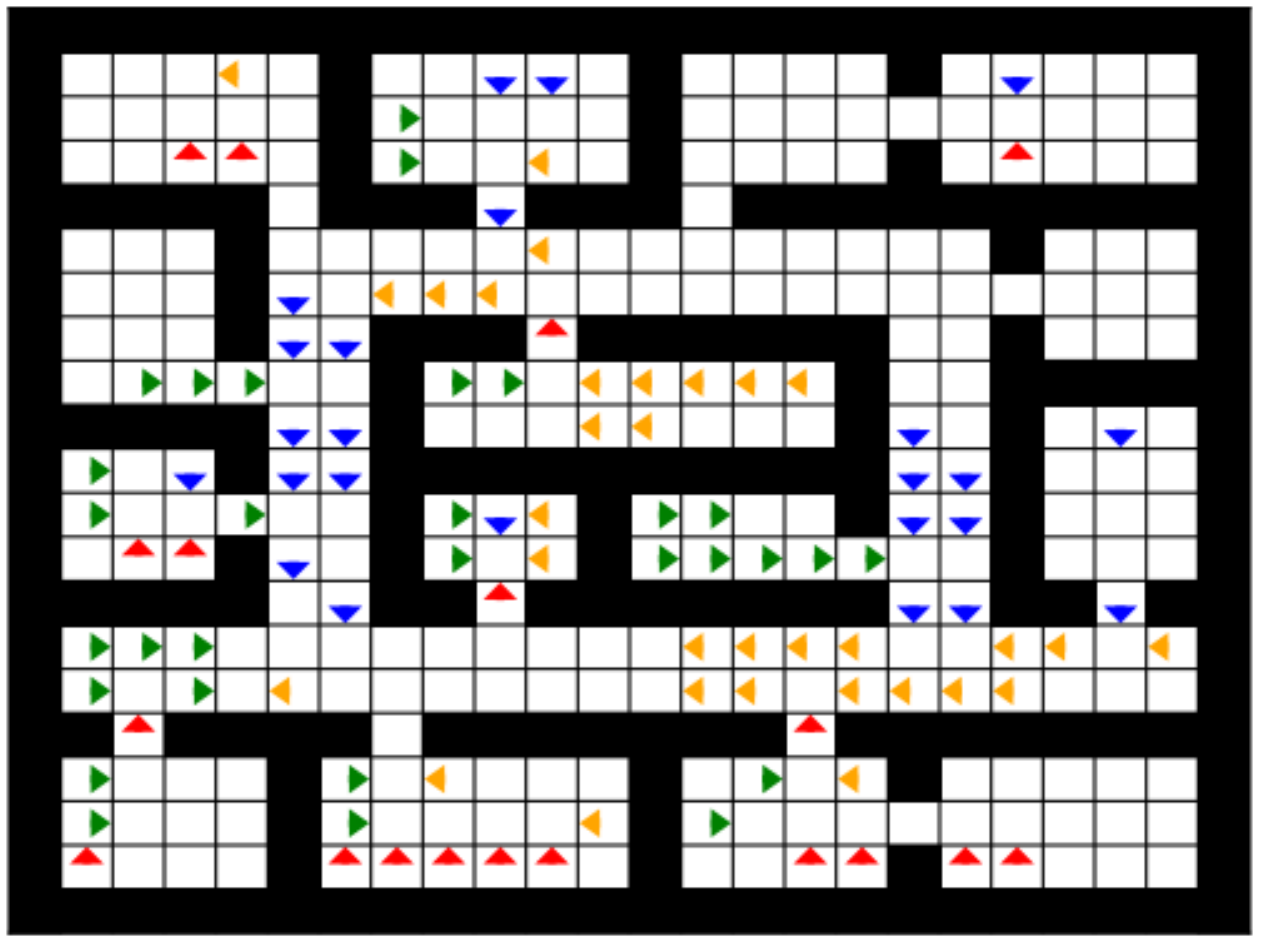}\caption{Identified desirable actions for the common reward environment in
Figure \ref{fig:commonrew_env}. }
\label{fig:commonrew_priors}
\end{figure}

\begin{figure}[H]
\centering{}\includegraphics[width=0.8\columnwidth]{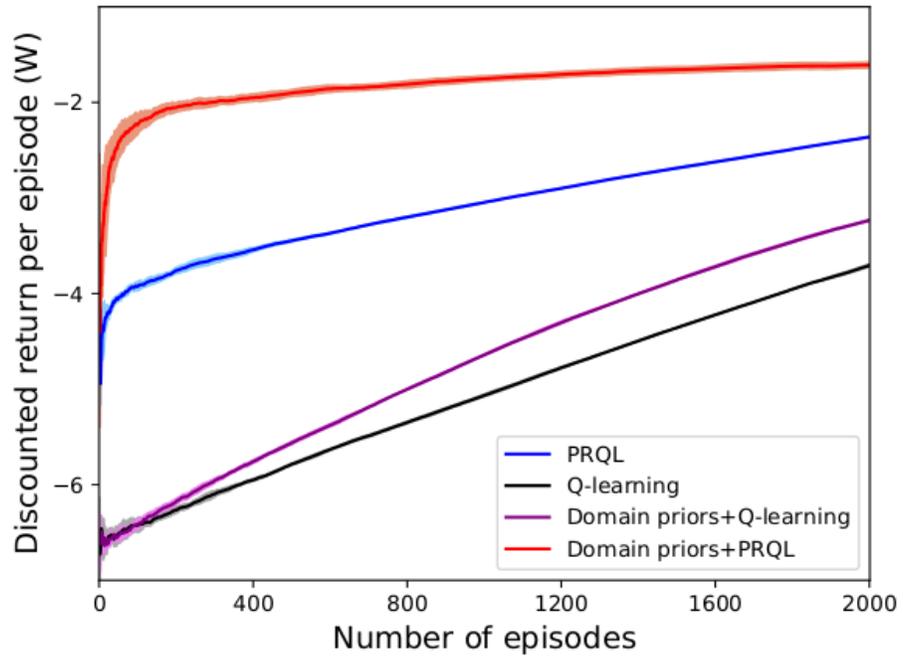}\caption{The average discounted return per episode ($W$), computed over $10$
trials, for different learning methods.}
\label{fig:performance_commonrew}
\end{figure}

\end{document}